\documentclass{article}
\usepackage[margin=1.2in,letterpaper]{geometry}

\usepackage{booktabs}
\usepackage{multirow}
\usepackage{hyperref}
\usepackage{xspace}
\usepackage{enumerate}
\usepackage{enumitem}
\usepackage{bm}
\usepackage{amssymb}
\usepackage{amsthm}
\usepackage{amsmath}
\usepackage{mathtools}
\usepackage{algorithm}
\usepackage{algorithmic}
\usepackage[numbers]{natbib}
\usepackage{url}

\newtheorem{theorem}{Theorem}
\newtheorem{corollary}{Corollary}

\theoremstyle{remark}

\newcommand{\x}{{\bf x}}
\newcommand{\y}{{\bf y}}

\newcommand{\A}{{\mathcal{A}}}
\newcommand{\C}{{\mathcal{C}}}

\renewcommand{\L}{{\mathcal{L}}}

\newcommand{\X}{{\mathcal{X}}}

\newcommand{\F}{{\mathcal{F}}}

\newcommand{\R}{\mathbb{R}}
\newcommand{\E}{\mathbb{E}}

\newcommand{\B}{\mathbb{B}}

\newcommand{\CH}{\text{CH}}
\newcommand{\Span}{\text{span}}
\newcommand{\bzero}{\mathbf{0}}
\newcommand{\Clin}{\L}
\newcommand{\Ccvx}{\C}
\newcommand{\ystar}{y^\star}
\newcommand{\supplementary}{supplementary material\xspace}
\def\half{\tfrac{1}{2}}

\begin{document}
\title{Online Gradient Boosting}
\author{Alina Beygelzimer \\ Yahoo Labs \\ New York, NY 10036 \\ \texttt{beygel@yahoo-inc.com}
\and
Elad Hazan \\ Princeton University \\ Princeton, NJ 08540 \\ \texttt{ehazan@cs.princeton.edu}
\and
Satyen Kale \\ Yahoo Labs \\ New York, NY 10036 \\ \texttt{satyen@yahoo-inc.com}
\and
Haipeng Luo \\ Princeton University \\ Princeton, NJ 08540 \\ \texttt{haipengl@cs.princeton.edu}}
\maketitle	

\begin{abstract} 
We extend the theory of boosting for regression problems to the online learning setting. Generalizing from the batch setting for boosting, the notion of a weak learning algorithm is modeled as an online learning algorithm with linear loss functions that competes with a base class of regression functions, while a strong learning algorithm is an online learning algorithm with smooth convex loss functions that competes with a larger class of regression functions. Our main result is an online gradient boosting algorithm that converts a weak online learning algorithm into a strong one where the larger class of functions is the linear span of the base class. We also give a simpler boosting algorithm that converts a weak online learning algorithm into a strong one where the larger class of functions is the convex hull of the base class, and prove its optimality.
\end{abstract} 

\section{Introduction}

Boosting algorithms~\citep{SchapireFr12} are ensemble methods that convert a learning algorithm for a base class of models with weak predictive power, such as decision trees, into a learning algorithm for a class of models with stronger predictive power, such as a weighted majority vote over base models in the case of classification, or a linear combination of base models in the case of regression. 

Boosting methods such as AdaBoost~\citep{FreundSc97} and Gradient Boosting~\citep{Friedman01} have found tremendous practical application, especially using decision trees as the base class of models. These algorithms were developed in the batch setting, where training is done over a fixed batch of sample data. However, with the recent explosion of huge data sets which do not fit in main memory, training in the batch setting is infeasible, and online learning techniques which train a model in one pass over the data have proven extremely useful. 

A natural goal therefore is to extend boosting algorithms to the online learning setting. Indeed, there has already been some work on online 
boosting for classification problems~\citep{OzaRu01, GrabnerBi06, LiuYu07, GrabnerLeBi08, ChenLiLu12, ChenLiLu14, BKL15}. Of these, the work by \citet{ChenLiLu12} provided the first theoretical study of online boosting for classification, which was later generalized by \citet{BKL15} to obtain optimal and adaptive online boosting algorithms.

However, extending boosting algorithms for regression to the online setting has been elusive and escaped theoretical guarantees thus far.  In this paper, we rigorously formalize the setting of online boosting for regression and then extend the very commonly used gradient boosting methods~\citep{Friedman01, MasonBaBaFr00} to the online setting, providing theoretical guarantees on their performance. 

The main result of this paper is an online boosting algorithm that competes with any linear combination the base functions, given an online linear learning algorithm over the base class. This algorithm is the online analogue of the batch boosting algorithm of \citet{ZhangYu05}, and in fact our algorithmic technique, when specialized to the batch boosting setting, provides exponentially better convergence guarantees. 

We also give an online boosting algorithm that competes with the best convex combination of base functions. This is a simpler algorithm which is analyzed along the lines of the Frank-Wolfe algorithm~\citep{FW}.
While the algorithm has weaker theoretical guarantees, it can still be useful in practice.  We also prove that this algorithm obtains the optimal regret bound (up to constant factors) for this setting.

Finally, we conduct some proof-of-concept experiments which show that our online boosting algorithms do obtain performance improvements over different classes of base learners.

\subsection{Related Work}

While the theory of boosting for classification in the batch setting is well-developed (see \citep{SchapireFr12}), the theory of boosting for regression is comparatively sparse.
The foundational theory of boosting for regression can be found in the statistics literature~\citep{HastieTiFr01, HastieTibshirani}, where boosting is understood as a greedy stagewise algorithm for fitting of additive models. The goal is to achieve the performance of linear combinations of base models, and to prove convergence to the performance of the best such linear combination.

 While the earliest works on boosting for regression such as \citep{Friedman01} do not have such convergence proofs, later works such as \citep{MasonBaBaFr00, CollinsScSi00} do have convergence proofs but without a bound on the speed of convergence. Bounds on the speed of convergence have been obtained by \citet{DuffyHe02} relying on a somewhat strong assumption on the performance of the base learning algorithm. A different approach to boosting for regression was taken by \citet{FreundSc97}, who give an algorithm that reduces the regression problem to classification and then applies AdaBoost; the corresponding proof of convergence relies on an assumption on the induced classification problem which may be hard to satisfy in practice. The strongest result is that of \citet{ZhangYu05}, who prove convergence to the performance of the best linear combination of base functions, along with a bound on the rate of convergence, making essentially no assumptions on the performance of the base learning algorithm. \citet{Telgarsky13} proves similar results for logistic (or similar) loss using a slightly simpler boosting algorithm. 

The results in this paper are a generalization of the results of \citet{ZhangYu05} to the online setting. However, we emphasize that this generalization is nontrivial and requires different algorithmic ideas and proof techniques. Indeed, we were not able to directly generalize the analysis in \citep{ZhangYu05} by simply adapting the techniques used in recent online boosting work \citep{ChenLiLu12, BKL15}, but we made use of the classical Frank-Wolfe algorithm \citep{FW}. On the other hand, while an important part of the convergence analysis for the batch setting  is to show statistical consistency of the algorithms \citep{ZhangYu05, BartlettTr07, Telgarsky13}, in the online setting we only need to study the empirical convergence (that is, the regret), which makes our analysis much more concise.

\section{Setup} 

Examples are chosen from a feature space $\X$, and the prediction space is $\R^d$. Let $\|\cdot\|$ denote some norm in $\R^d$. In the setting for online regression, in each round $t$ for $t = 1, 2, \ldots, T$, an adversary selects an example $\x_t \in \X$ and a loss function $\ell_t: \R^d \rightarrow \R$, and presents $\x_t$ to the online learner. The online learner outputs a prediction $\y_t \in \R^d$, obtains the loss function $\ell_t$, and incurs loss $\ell_t(\y_t)$. 

Let $\F$ denote a reference class of regression functions $f: \X \rightarrow \R^d$, and let $\C$ denote a class of loss functions $\ell: \R^d \rightarrow \R$. Also, let $R: \mathbb{N} \rightarrow \R_+$ be a non-decreasing function. We say that the function class $\F$ is {\em online learnable} for losses in $\C$ with regret $R$ if there is an online learning algorithm $\A$, that for every $T \in \mathbb{N}$ and every sequence $(\x_t, \ell_t) \in \X \times \C$ for $t = 1, 2, \ldots, T$ chosen by the adversary, generates predictions\footnote{There is a slight abuse of notation here. $\A(\cdot)$ is not a function but rather the output of the online learning algorithm $\A$ computed on the given example using its internal state.} $\A(\x_t) \in \R^d$ such that
\begin{equation} \label{eq:regret}
	\sum_{t=1}^T \ell_t(\A(\x_t))\ \leq\ \inf_{f \in \F} \sum_{t=1}^T \ell_t(f(\x_t)) + R(T).
\end{equation}
If the online learning algorithm is randomized, we require the above bound to hold with high probability.

The above definition is simply the online generalization of standard empirical risk minimization (ERM) in the batch setting. A concrete example is $1$-dimensional regression, i.e. the prediction space is $\R$. For a labeled data point $(\x, \ystar) \in \X \times \R$, the loss for the prediction $y \in \R$ is given by $\ell(\ystar, y)$ where $\ell(\cdot, \cdot)$ is a fixed loss function that is convex in the second argument (such as squared loss, logistic loss, etc). Given a batch of $T$ labeled data points $\{(\x_t, \ystar_t)\ |\ t = 1, 2, \ldots, T\}$ and a base class of regression functions $\F$ (say, the set of bounded norm linear regressors), an ERM algorithm finds the function $f \in \F$ that minimizes $\sum_{t=1}^T \ell(\ystar_t, f(\x_t))$. 

In the online setting, the adversary reveals the data $(\x_t, \ystar_t)$ in an online fashion, only presenting the true label $\ystar_t$ after the online learner $\A$ has chosen a prediction $y_t$. Thus, setting $\ell_t(y_t) = \ell(\ystar_t, y_t)$, we observe that if $\A$ satisfies the regret bound (\ref{eq:regret}), then it makes predictions with total loss almost as small as that of the empirical risk minimizer, up to the regret term. If $\F$ is the set of all bounded-norm linear regressors, for example, the algorithm $\A$ could be online gradient descent~\citep{Zinkevich03} or online Newton Step~\citep{HazanAgKa07}.

At a high level, in the batch setting, ``boosting'' is understood as a procedure that, given a batch of data and access to an ERM algorithm for a function class $\F$ (this is called a ``weak'' learner), obtains an approximate ERM algorithm for a richer function class $\F'$ (this is called a ``strong'' learner). Generally, $\F'$ is the set of finite linear combinations of functions in $\F$. The efficiency of boosting is measured by how many times, $N$, the base ERM algorithm needs to be called (i.e., the number of boosting steps) to obtain an ERM algorithm for the richer function within the desired approximation tolerance. Convergence rates \citep{ZhangYu05} give bounds on how quickly the approximation error goes to $0$ and $N \rightarrow \infty$.

We now extend this notion of boosting to the online setting in the natural manner. To capture the full generality of the techniques, we also specify a class of loss functions that the online learning algorithm can work with. Informally, an online boosting algorithm is a reduction that, given access to an online learning algorithm $\A$ for a function class $\F$ and loss function class $\C$ with regret $R$, and a bound $N$ on the total number of calls made in each iteration to copies of $\A$, obtains an online learning algorithm $\A'$ for a richer function class $\F'$, a richer loss function class $\C'$, and (possibly larger) regret $R'$. The bound $N$ on the total number of calls made to all the copies of $\A$ corresponds to the number of boosting stages in the batch setting, and in the online setting it may be viewed as a resource constraint on the algorithm. The efficacy of the reduction is measured by $R'$ which is a function of $R$, $N$, and certain parameters of the comparator class $\F'$ and loss function class $\C'$. We desire online boosting algorithms such that $\frac{1}{T}R'(T) \rightarrow 0$ quickly as $N \rightarrow \infty$ and $T \rightarrow \infty$. We make the notions of richness in the above informal description more precise now. 

\paragraph{Comparator function classes.}
A given function class $\F$ is said to be $D$-bounded if for all $\x \in \X$ and all $f \in \F$, we have $\|f(\x)\| \leq D$. Throughout this paper, we assume that $\F$ is symmetric:\footnote{This is without loss of generality; as will be seen momentarily, our base assumption only requires an online learning algorithm $\A$ for $\F$ for linear losses $\ell_t$. By running the Hedge algorithm on two copies of $\A$, one of which receives the actual loss functions $\ell_t$ and the other recieves $-\ell_t$, we get an algorithm which competes with negations of functions in $\F$ and the constant zero function as well. Furthermore, since the loss functions are convex (indeed, linear) this can be made into a deterministic reduction by choosing the convex combination of the outputs of the two copies of $\A$ with mixing weights given by the Hedge algorithm.} i.e. if $f \in \F$, then $-f \in \F$, and it contains the constant zero function, which we denote, with some abuse of notation, by $\bzero$.

Given $\F$, we define two richer function classes $\F'$: the convex hull of $\F$, denoted $\CH(\F)$, is the set of convex combinations of a finite number of functions in $\F$, and the span of $\F$, denoted $\Span(\F)$, is the set of linear combinations of finitely many functions in $\F$. For any $f \in \Span(\F)$, define $\|f\|_1 := \inf \left\{\max\{1, \sum_{g \in S} |w_g|\}:\ f = \sum_{g \in S} w_g g,\ S \subseteq \F,\ |S| < \infty,\  w_g \in \R\right\}$. Since functions in $\Span(\F)$ are not bounded, it is not possible to obtain a uniform regret bound for all functions in $\Span(\F)$: rather, the regret of an online learning algorithm $\A$ for $\Span(\F)$ is specified in terms of regret bounds for individual comparator functions $f \in \Span(F)$, viz.
\[ R_f(T)\ :=\ \sum_{t=1}^T \ell_t(\A(\x_t)) - \sum_{t=1}^T \ell_t(f(\x_t)).\]

\paragraph{Loss function classes.}
The base loss function class we consider is $\Clin$, the set of all linear functions $\ell: \R^d \rightarrow \R$, with Lipschitz constant bounded by $1$. A function class $\F$ that is online learnable with the loss function class $\Clin$ is called {\em online linear learnable} for short. The richer loss function class we consider is denoted by $\Ccvx$ and is a set of convex loss functions $\ell: \R^d \rightarrow \R$ satisfying some regularity conditions specified in terms of certain parameters described below.

We define a few parameters of the class $\Ccvx$. For any $b > 0$, let $\B^d(b) = \{\y \in \R^d:\ \|\y\| \leq b\}$ be the ball of radius $b$. The class $\Ccvx$ is said to have Lipschitz constant $L_b$ on $\B^d(b)$ if for all $\ell \in \Ccvx$ and all $\y \in \B^d(b)$ there is an efficiently computable subgradient $\nabla \ell(\y)$ with norm at most $L_b$. Next, $\Ccvx$ is said to be $\beta_b$-smooth on $\B^d(b)$ if for all $\ell \in \Ccvx$ and all $\y, \y' \in \B^d(b)$ we have
\[ \ell(\y')\ \leq\ \ell(\y) + \nabla \ell(\y) \cdot (\y' - \y) + \frac{\beta_b}{2}\|\y - \y'\|^2.\]
Next, define the projection operator $\Pi_b: \R^d \rightarrow \B^d(b)$ as 
$\Pi_b(\y) := \arg\min_{\y' \in \B^d(b)} \|\y - \y'\|$, and define 
$\epsilon_b := \sup_{\y \in \R^d,\ \ell \in \Ccvx} \frac{\ell(\Pi_b(\y)) - \ell(\y)}{\|\Pi_b(\y) - \y\|}$.

\section{Online Boosting Algorithms} 

The setup is that we are given a $D$-bounded reference class of functions $\F$ with an online linear learning algorithm $\A$ with regret bound $R(\cdot)$. For normalization, we also assume that the output of $\A$ at any time is bounded in norm by $D$, i.e. $\|\A(\x_t)\| \leq D$ for all $t$. We further assume that for every $b > 0$, we can compute\footnote{It suffices to compute upper bounds on these parameters.} a Lipschitz constant $L_b$, a smoothness parameter $\beta_b$, and the parameter $\epsilon_b$ for the class $\Ccvx$ over $\B^d(b)$. Furthermore, the online boosting algorithm may make up to $N$ calls per iteration to any copies of $\A$ it maintains, for a given a budget parameter $N$.

\begin{algorithm}[h]
\caption{Online Gradient Boosting for $\Span(\F)$}
\label{alg:spanboost}
\begin{algorithmic}[1]
\REQUIRE Number of weak learners $N$, step size parameter $\eta \in [\frac{1}{N}, 1]$, 
\STATE Let $B = \min\{\eta N D,\ \inf \{b \geq D:\ \eta \beta_b b^2 \geq \epsilon_b D\}\}$.
\STATE Maintain $N$ copies of the algorithm $\A$, denoted $\A^i$ for $i = 1, 2, \ldots, N$. 
\STATE For each $i$, initialize $\sigma^i_1 = 0$.
\FOR{$t=1$ {\bfseries to} $T$}
    \STATE Receive example $\x_t$.
    \STATE Define $\y_t^0 = \bzero$.
    \FOR{$i=1$ {\bfseries to} $N$}
    	\STATE Define $\y_t^i = \Pi_B((1 - \sigma^i_t\eta)\y_t^{i-1} + \eta \A^i(\x_t))$.
    \ENDFOR
    \STATE Predict $\y_t = \y_t^N$. 
    \STATE Obtain loss function $\ell_t$ and suffer loss $\ell_t(\y_t)$.
    \FOR{$i=1$ {\bfseries to} $N$}
    	\STATE Pass loss function  $\ell_t^i(\y) = \frac{1}{L_B}\nabla \ell_t(\y_t^{i-1}) \cdot \y$ to $\A^i$.
        \STATE Set $\sigma^i_{t+1} = \max\{\min\{\sigma^i_t + \alpha_t \nabla \ell_t(\y_t^{i-1}) \cdot \y_t^{i-1}), 1\}, 0\}$, where $\alpha_t = \frac{1}{L_BB\sqrt{t}}$.
    \ENDFOR
\ENDFOR
\end{algorithmic}
\end{algorithm}

Given this setup, our main result is an online boosting algorithm, Algorithm~\ref{alg:spanboost}, competing with $\Span(\F)$. The algorithm maintains $N$ copies of $\A$, denoted $\A^i$, for $i = 1, 2, \ldots, N$. Each copy corresponds to one stage in boosting. When it receives a new example $\x_t$, it passes it to each $\A^i$ and obtains their predictions $\A^i(\x_t)$, which it then combines into a prediction for $\y_t$ using a linear combination. At the most basic level, this linear combination is simply the sum of all the predictions scaled by a step size parameter $\eta$. Two tweaks are made to this sum in step 8 to facilitate the analysis: 
\begin{enumerate}
	\item While constructing the sum, the partial sum $\y_t^{i-1}$ is multiplied by a {\em shrinkage} factor $(1 - \sigma_t^i \eta)$. This shrinkage term is tuned using an online gradient descent algorithm in step 14. The goal of the tuning is to induce the partial sums $\y_t^{i-1}$ to be aligned with a descent direction for the loss functions, as measured by the inner product $\nabla \ell_t(\y_t^{i-1}) \cdot y_t^{i-1}$.

	\item The partial sums $\y_t^i$ are made to lie in $\B^d(B)$, for some parameter $B$, by using the projection operator $\Pi_B$. This is done to ensure that the Lipschitz constant and smoothness of the loss function are suitably bounded.
\end{enumerate}
Once the boosting algorithm makes the prediction $\y_t$ and obtains the loss function $\ell_t$, each $\A^i$ is updated using a suitably scaled linear approximation to the loss function at the partial sum $\y_t^{i-1}$, i.e. the linear loss function $\frac{1}{L_B}\nabla \ell_t(\y_t^{i-1}) \cdot \y$. This forces $\A^i$ to produce predictions that are aligned with a descent direction for the loss function.

We provide the analysis of the algorithm in Section~\ref{sec:span}. The analysis yields the following regret bound for the algorithm:
\begin{theorem} \label{thm:span}
	Let $\eta \in [\frac{1}{N}, 1]$ be a given parameter. Let $B = \min\{\eta N D,\ \inf \{b \geq D:\ \eta \beta_b b^2 \geq \epsilon_b D\}\}$. Algorithm~\ref{alg:spanboost} is an online learning algorithm for $\Span(\F)$ and losses in $\Ccvx$ with the following regret bound for any $f \in \Span(\F)$:
	\[ R_f'(T)\ \leq\ \left(1 - \frac{\eta}{\|f\|_1}\right)^N\!\!\!\Delta_0 + 3\eta\beta_B B^2 \|f\|_1 T + L_B \|f\|_1R(T) + 2L_B B \|f\|_1\sqrt{T},\]
	where $\Delta_0 := \sum_{t=1}^T \ell_t(\bzero) - \ell_t(f(\x_t))$.
\end{theorem}
The regret bound in this theorem depends on several parameters such as $B$, $\beta_B$ and $L_B$. In applications of the algorithm for $1$-dimensional regression with commonly used loss functions, however, these parameters are essentially modest constants; see Section~\ref{sec:params} for calculations of the parameters for various loss functions. Furthermore, if $\eta$ is appropriately set (e.g. $\eta = (\log N)/N$), then the average regret $R'_f(T)/T$ clearly converges to $0$ as $N \rightarrow \infty$ and $T \rightarrow \infty$. While the requirement that $N \rightarrow \infty$ may raise concerns about computational efficiency, this is in fact analogous to the guarantee in the batch setting: the algorithms converge only when the number of boosting stages goes to infinity. Moreover, our lower bound (Theorem~\ref{thm:lb}) shows that this is indeed necessary. 

We also present a simpler boosting algorithm, Algorithm~\ref{alg:chboost}, that competes with $\CH(\F)$. Algorithm~\ref{alg:chboost} is similar to Algorithm~\ref{alg:spanboost}, with some simplifications: the final prediction is simply a convex combination of the predictions of the base learners, with no projections or shrinkage necessary. While Algorithm~\ref{alg:spanboost} is more general, Algorithm~\ref{alg:chboost} may still be useful in practice when a bound on the norm of the comparator function is known in advance, using the observations in Section~\ref{sec:main-scaling}. Furthermore, its analysis is cleaner and easier to understand for readers who are familiar with the Frank-Wolfe method, and this serves as a foundation for the analysis of Algorithm~\ref{alg:spanboost}. This algorithm has an optimal (up to constant factors) regret bound as given in the following theorem, proved in Section~\ref{sec:ch}. The upper bound in this theorem is proved along the lines of the Frank-Wolfe~\citep{FW} algorithm, and the lower bound using information-theoretic arguments.
\begin{theorem} \label{thm:ch} 
	Algorithm~\ref{alg:chboost} is an online learning algorithm for $\CH(\F)$ for losses in $\Ccvx$  with the regret bound
	\[ R'(T)\ \leq\ \frac{8\beta_D D^2}{N}T + L_DR(T).\]
	Furthermore, the dependence of this regret bound on $N$ is optimal up to constant factors.
\end{theorem}
The dependence of the regret bound on $R(T)$ is unimprovable without additional assumptions: otherwise, Algorithm \ref{alg:chboost} will be an online linear learning algorithm over $\F$ with better than $R(T)$ regret.

\begin{algorithm}[h]
\caption{Online Gradient Boosting for $\CH(\F)$}
\label{alg:chboost}
\begin{algorithmic}[1]
\STATE Maintain $N$ copies of the algorithm $\A$, denoted $\A^1, \A^2, \ldots, \A^N$, and let $\eta_i = \frac{2}{i+1}$ for $i = 1, 2, \ldots, N$.
\FOR{$t=1$ {\bfseries to} $T$}
    \STATE Receive example $\x_t$.
    \STATE Define $\y_t^0 = \bzero$.
    \FOR{$i=1$ {\bfseries to} $N$}
    	\STATE Define $\y_t^i = (1 - \eta_i) \y_t^{i-1} + \eta_i \A^i(\x_t)$.
    \ENDFOR
    \STATE Predict $\y_t = \y_t^N$. 
    \STATE Obtain loss function $\ell_t$ and suffer loss $\ell_t(\y_t)$.
    \FOR{$i=1$ {\bfseries to} $N$}
        \STATE Pass loss function $\ell_t^i(\y) = \frac{1}{L_D}\nabla \ell_t(\y_t^{i-1}) \cdot \y$ to $\A^i$.
    \ENDFOR
\ENDFOR
\end{algorithmic}
\end{algorithm}

\paragraph{Using a deterministic base online linear learning algorithm.} If the base online linear learning algorithm $\A$ is deterministic, then our results can be improved, because our online boosting algorithms are also deterministic, and using a standard simple reduction, we can now allow $\Ccvx$ to be any set of convex functions (smooth or not) with a computable Lipschitz constant $L_b$ over the domain $\B^d(b)$ for any $b > 0$.

This reduction converts arbitrary convex loss functions into linear functions: viz. if $\y_t$ is the output of the online boosting algorithm, then the loss function provided to the boosting algorithm as feedback is the linear function $\ell'_t(\y) = \nabla \ell_t(\y_t) \cdot \y$. This reduction immediately implies that the base online linear learning algorithm $\A$, when fed loss functions $\frac{1}{L_D}\ell'_t$, is already an online learning algorithm for $\CH(\F)$ with losses in $\Ccvx$ with the regret bound $R'(T) \leq L_DR(T)$.

As for competing with $\Span(\F)$, since linear loss functions are $0$-smooth, we obtain the following easy corollary of Theorem~\ref{thm:span}: 
\begin{corollary} \label{cor:span}
	Let $\eta \in [\frac{1}{N}, 1]$ be a given parameter, and set $B = \eta ND$. Algorithm~\ref{alg:spanboost} is an online learning algorithm for $\Span(\F)$ for losses in $\Ccvx$ with the following regret bound for any $f \in \Span(\F)$:
	\[ R_f'(T)\ \leq\ \left(1 - \frac{\eta}{\|f\|_1}\right)^N\!\!\!\Delta_0 + L_B \|f\|_1R(T) + 2L_B B \|f\|_1\sqrt{T},\]
	where $\Delta_0 := \sum_{t=1}^T \ell_t(\bzero) - \ell_t(f(\x_t))$.
\end{corollary}

\subsection{The parameters for several basic loss functions}
\label{sec:params}

In this section we consider the application of our results to $1$-dimensional regression, where we assume, for normalization, that the true labels of the examples and the predictions of the functions in the class $\F$ are in $[-1, 1]$. In this case $\|\cdot\|$ denotes the absolute value norm. Thus, in each round, the adversary chooses a labeled data point $(\x_t, \ystar_t) \in \X \times [-1, 1]$, and the loss for the prediction $y_t \in [-1, 1]$ is given by $\ell_t(y_t) = \ell(\ystar_t, y_t)$ where $\ell(\cdot, \cdot)$ is a fixed loss function that is convex in the second argument. Note that $D = 1$ in this setting. We give examples of several such loss functions below, and compute the parameters $L_b$, $\beta_b$ and $\epsilon_b$ for every $b > 0$, as well as $B$ from Theorem~\ref{thm:span}.
\begin{enumerate}
	\item Linear loss: $\ell(\ystar, y) = -\ystar y$. We have $L_b = 1$, $\beta_b = 0$, $\epsilon_b = 1$, and $B = \eta N$.

	\item $p$-norm loss, for some $p \geq 2$: $\ell(\ystar, y) = |\ystar - y|^p$. We have $L_b = p(b + 1)^{p-1}$, $\beta_b = p(p-1)(b+1)^{p-2}$, $\epsilon_b = \max\{p(1 - b)^{p-1}, 0\}$, and $B = 1$.

	\item Modified least squares: $\ell(\ystar, y) = \frac{1}{2}\max\{1 - \ystar y, 0\}^2$. 
	We have $L_b = b + 1$, $\beta_b = 1$, $\epsilon_b = \max\{1 - b, 0\}$, and $B = 1$.

	\item Logistic loss: $\ell(\ystar, y) = \ln(1 + \exp(-\ystar y))$. 
	We have $L_b = \frac{\exp(b)}{1 + \exp(b)}$, $\beta_b = \frac{1}{4}$, $\epsilon_b = \frac{\exp(-b)}{1 + \exp(-b)}$, and $B = \min\{\eta N, \ln(4/\eta)\}$.
\end{enumerate}

\section{Analysis}
In this section, we analyze Algorithms~\ref{alg:spanboost} and Algorithm~\ref{alg:chboost}.

\subsection{Competing with convex combinations of the base functions}
\label{sec:ch}

We give the analysis of Algorithm~\ref{alg:chboost} before that of Algorithm~\ref{alg:spanboost} since it is easier to understand and provides the foundation for the analysis of Algorithm~\ref{alg:spanboost}.

\begin{proof}[Proof of Theorem~\ref{thm:ch}]
	First, note that for any $i = 1, 2, \ldots, N$, since $\ell_t^i$ is a linear function, we have
	\[ \inf_{f \in \CH(\F)} \sum_{t=1}^T \ell_t^i(f(\x_t))\ =\ \inf_{f \in \F} \sum_{t=1}^T \ell_t^i(f(\x_t)).  \]
	Let $f$ be any function in $\CH(\F)$. The equality above and the fact that $\A^i$ is an online learning algorithm for $\F$ with regret bound $R(\cdot)$ for the $1$-Lipschitz linear loss functions $\ell_t^i(\y) = \frac{1}{L_D}\nabla \ell_t(\y_t^{i-1}) \cdot \y$ imply that
	\begin{equation} \label{eq:wl-stepwise}
		\sum_{t=1}^T \frac{1}{L_D}\nabla \ell_t(\y_t^{i-1}) \cdot \A^i(\x_t)\ \leq\ \sum_{t=1}^T \frac{1}{L_D}\nabla \ell_t(\y_t^{i-1}) \cdot f(\x_t)  + R(T).
	\end{equation}
	Now define, for $i = 0, 1, 2, \ldots, N$, $\Delta_i = \sum_{t=1}^T \ell_t(\y_t^i) - \ell_t(f(\x_t))$. We have
	\begin{align*}
		\Delta_i\ &=\ \sum_{t=1}^T \ell_t(\y_t^{i-1} + \eta_i(\A^i(\x_t) - \y_t^{i-1})) - \ell_t(f(\x_t))\\
		&\leq\ \sum_{t=1}^T \ell_t(\y_t^{i-1}) - \ell_t(f(\x_t)) + \eta_i \nabla \ell_t(\y_t^{i-1}) \cdot (\A^i(\x_t) - \y_t^{i-1})  + \frac{\eta_i^2\beta_D}{2} \| \A^i(\x_t) - \y_t^{i-1} \|^2 \\
		& (\text{by $\beta_D$-smoothness of $\Ccvx$})\\
		&\leq\ \left[\sum_{t=1}^T \ell_t(\y_t^{i-1}) - \ell_t(f(\x_t)) + \eta_i\nabla \ell_t(\y_t^{i-1}) \cdot (f(\x_t) - \y_t^{i-1})  + 2\eta_i^2\beta_D D^2\right] + \eta_iL_DR(T) \\
		&(\text{by (\ref{eq:wl-stepwise}) and using the bound $\| \A^i(\x_t) - \y_t^{i-1} \| \leq 2D$}) \\
		&\leq \left[\sum_{t=1}^T \ell_t(\y_t^{i-1}) - \ell_t(f(\x_t)) - \eta_i(\ell_t(\y_t^{i-1}) - \ell_t(f(\x_t))) + 2\eta_i^2\beta_D D^2\right] + \eta_iL_DR(T) \\
		& \left(\text{by convexity, } \ell_t(\y_t^{i-1}) +  \nabla \ell(\y_t^{i-1}) \cdot (f(\x_t) - \y_t^{i-1}) \leq \ell_t(f(\x_t))\right)\\
		&\leq\ (1 - \eta_i)\Delta_{i-1} + 2\eta_i^2 \beta_D D^2 T + \eta_i L_DR(T).\\
	\end{align*}
	For $i = 1$, since $\eta_1 = 1$, the above bound implies that $\Delta_1 \leq 2\beta_D D^2 T + L_DR(T)$. Starting from this base case, an easy induction on $i \geq 1$ proves that $\Delta_i \leq \frac{8\beta_D D^2}{i}T + L_DR(T)$. Applying this bound for $i = N$ completes the proof.
\end{proof}

We now show that the dependence of the regret bound of Algorithm~\ref{alg:chboost} on the parameter $N$ is optimal up to constant factors.

\begin{theorem} \label{thm:lb}
	Let $N$ be any specified bound on the total number of calls in each iteration to all copies of the base online linear learning algorithm. Then there is a setting of $1$-dimensional prediction with a $1$-bounded comparator function class $\F$, an online linear optimization algorithm $\A$ over $\F$, and a class $\Ccvx$ of loss functions that is $1$-smooth on $\R$ such that any online boosting algorithm for $\CH(\F)$ with losses in $\Ccvx$ respecting the bound $N$ has regret at least $\Omega(\frac{T}{N})$.
\end{theorem}
\begin{proof} Consider the following construction. At a high level, the setting is 1-dimensional regression with $\Ccvx$ corresponding to squared loss. The domain $\X = \mathbb{N}$ and true labels of examples are in $[0, 1]$.

Define $p_1 = \half + \epsilon$ and $p_2 = \half - \epsilon$, where $\epsilon = \frac{1}{10\sqrt{N}}$, and let $D_1$ and $D_2$ be two distributions over $\{0, 1\}^N$ where each bit is Bernoulli random variable with parameter $p_1$ and $p_2$ respectively, chosen independently of the other bits. Consider a sequence of examples $(\x_t, \ystar_t) \in \mathbb{N} \times [0, 1]$ generated as follows: $\x_t  = t$, and the label $\ystar_t$ is chosen from $\{p_1, p_2\}$ uniformly at random in each round.

Let for $c = \frac{1}{4000}$. The function class $\F$ consists of a large number, $M = \frac{1}{c}N$, of functions $f_i$, $i \in [M]$. For each $t$ and $i$, we set $f_i(\x_t) = 1$ w.p. $\ystar_t$, and $0$ w.p. $1 - \ystar_t$, independently of all other values of $t$ and $i$. 

The base online linear learning algorithm $\A$ is simply Hedge over the $M$ functions. In each round, the Hedge algorithm selects one of the $M$ functions in $\F$ and uses that to predict the label, and for any sequence of $T$ examples, with high probability, incurs regret $R(T) = O(\sqrt{\log(M) T})$. 

We set $\Ccvx$ to be set of squared loss functions, i.e. functions of the form $\ell(y) = \half(y - \ystar)^2$ for $\ystar \in [0, 1]$. Note that these loss functions are $1$-smooth and $D = 1$. In round $t$, the loss function is $\ell_t(y) = \half(y - \ystar_t)^2$.

Consider the function $\bar{f} = \frac{1}{M}\sum_{i \in [M]} f_i$, which is in $\CH(\F)$. Given any input sequence $(\x_t, \ystar_t)$ for $t = 1, 2, \ldots, T$ it is easy to calculate that $\E[\half(\bar{f}(\x_t) - \ystar_t)^2] = \frac{\ystar_t(1 - \ystar_t)}{2M} \leq \frac{1}{2M}$, and since the examples and predictions of functions on the examples are independent across iterations, a simple application of the multiplicative Chernoff bound implies that if $T \geq 12M$, then with probability at least $0.9$, we have $\sum_{t=1}^T \half(\bar{f}(\x_t) - \ystar_t)^2\ \leq\ \frac{T}{M}$.

Now suppose there is an online boosting algorithm making at most $N$ calls total to all copies of $\A$ in each iteration, that for any large enough $T$ and for any sequence $(\x_t, \ystar_t)$ for $t = 1, 2, \ldots, T$, outputs predictions $y_t$ such that with high probability, say at least $0.9$, we have $\sum_{t=1}^T \half(y_t - \ystar_t)^2\ \leq\ \sum_{t=1}^T \half(\bar{f}(\x_t) - \ystar_t)^2 + \frac{cT}{N}$. Then by a union bound, with probability at least $0.8$, we have $ \sum_{t=1}^T \half(y_t - \ystar_t)^2 \leq \frac{cT}{N} + \frac{T}{M} \leq \frac{2cT}{N}$. By Markov's inequality and a union bound, with probability at least $0.7$, for a uniform random time $\tau \in [T]$, we have 
\begin{equation} \label{eq:identification}
	\frac{1}{2}(y_\tau - \ystar_\tau)^2\ \leq\ \frac{20c}{N}\ =\ \frac{\epsilon^2}{2},
\end{equation}
or in other words, $y_\tau$ is on the same side of $\half$ as $\ystar_\tau$, and thus can be used to identify $\ystar_\tau$. In the rest of the proof, we will use this fact, along with fact the total variation distance between $D_1$ and $D_2$, denoted $d_\text{TV}(D_1, D_2)$, is small, to derive a contradiction.

Define the random variable $Y: \{0, 1\}^N \rightarrow \R$ as follows. For any bit string $s = \langle s_1, s_2, \ldots, s_N \rangle \in \{0, 1\}^N$, choose a random round $\tau \in [T]$, and simulate the online boosting process until round $\tau-1$ by sampling $\ystar_t$'s and the outputs of $f_i(\x_t)$ for all $t \leq \tau - 1$ and $i \in [M]$ from the appropriate distributions. In round $\tau$, let $f_{i_1}, f_{i_2}, \ldots, f_{i_N}$ be the functions that are obtained from the at most $N$ calls to copies of $\A$ (there could be repetitions). Assign $f_{i_j}(\x_\tau) = s_j$ for $j \in [N]$ (being careful with repeated functions and repeating outputs appropriately), and run the booster with these outputs to obtain $y_\tau$, and set $Y(s) = y_\tau$. Let $\Pr[\cdot]$ denotes probability of events in this process for generating $Y(s)$ given $s$. 

Let $\E_1[X(s)]$ and $\E_2[X(s)]$ denote expectation of a random variable $X: \{0, 1\}^N \rightarrow \R$ when $s$ is drawn from $D_1$ and $D_2$ respectively, and let $\E_0[X(I, s)]$ denote expectation of a random variable $X: \{1, 2\} \times \{0, 1\}^N \rightarrow \R$ when $I$ is chosen from $\{1, 2\}$ uniformly at random and then $s$ is sampled from $D_I$. The above analysis (inequality (\ref{eq:identification})) implies that 
\[0.7\ \leq\ \E_0[\Pr[|Y(s) - p_I| \leq \epsilon]]\ =\ \half \E_1[\Pr[|Y(s) - p_1| \leq \epsilon]] + \half \E_2[\Pr[|Y(s) - p_2| \leq \epsilon]].\]

Now define a random variable $X: \{0, 1\}^N \rightarrow \R$ as $X(s) = \Pr[Y(s) \geq \half]$. Since 
\[\Pr[Y(s) \geq \half]\ \geq\ \Pr[|Y(s) - p_1| \leq \epsilon]\  \text{ and }\ 1 - \Pr[Y(s) \geq \half]\ \geq\ \Pr[|Y(s) - p_2| \leq \epsilon],\] 
we conclude, using the above bound, that $\E_1[X(s)] - \E_2[X(s)] \geq 0.4$. This is a contradiction, since because $X(s) \in [0, 1]$, we have
\[ \E_1[X(s)] - \E_2[X(s)]\ \leq\ d_\text{TV}(D_1, D_2)\ <\ 4\sqrt{\epsilon^2 N}\ =\ 0.4,\]
where the bound on $d_\text{TV}(D_1, D_2)$ is standard, for e.g. see \cite{HKsvm}. This gives us the desired contradiction.
\end{proof}

The above result can be easily extended to any given parameters $\beta$ and $D$ so that the $\F$ is $D$-bounded and $\Ccvx$ is $\beta$-smooth on $\R$, giving a lower bound of $\Omega(\frac{\beta D^2T}{N})$ on the regret of an online boosting algorithm for $\CH(\F)$ with losses in $\Ccvx$: we simply scale all function and label values by $D$, and consider the loss functions $\ell(y, \ystar) = \frac{\beta}{2}(y - \ystar)^2$. If there were an online boosting algorithm for $\CH(\F)$ with these loss functions with regret $o(\frac{\beta D^2T}{N})$, then by scaling down the predictions by $D$, we obtain an online boosting algorithm for exactly the setting in the proof of Theorem~\ref{thm:lb} with a regret bound of $o(\frac{T}{N})$, which is a contradiction.

\subsection{Competing with the span of the base functions}
\label{sec:span}

In this section we show that Algorithm~\ref{alg:spanboost} satisfies the regret bound claimed in Theorem~\ref{thm:span}.
\begin{proof}[Proof of Theorem~\ref{thm:span}]
	Let $f = \sum_{g \in S} w_g g$, for some finite subset $S$ of $\F$, where $w_g \in \R$. Since $\F$ is symmetric, we may assume that all $w_g \geq 0$, and let $W := \sum_g w_g$. Furthermore, we may assume that $\bzero \in S$ with weight $w_\bzero = \max\{1 - \sum_{g \in S,\ g \neq \bzero} w_g, 0\}$, so that $W \geq 1$. Note that $\|f\|_1$ is exactly the infimum of $W$ over all such ways of expressing $f$ as a finite weighted sum of functions in $\F$. We now prove that bound stated in the theorem holds with $\|f\|_1$ replaced by $W$; the theorem then follows simply by taking the infimum of the bound over all such ways of expressing $f$.

	Now, for each $i \in [N]$, the update in line 14 of Algorithm~\ref{alg:spanboost} is exactly online gradient descent~\citep{Zinkevich03} on the domain $[0, 1]$ with linear loss functions $\sigma \mapsto -\nabla \ell_t(\y_t^{i-1}) \cdot \y_t^{i-1}\sigma$. Note that the derivative of this loss function is bounded as follows: $|-\nabla \ell_t(\y_t^{i-1}) \cdot \y_t^{i-1}| \leq L_B B$. Since $\frac{1}{W} \in [0, 1]$, the standard analysis of online gradient descent then implies that the sequence $\sigma_t^i$ for $t = 1, 2, \ldots, T$ satisfies
	\begin{equation} \label{eq:ogd-regret}
		\sum_{t=1}^T -\nabla \ell_t(\y_t^{i-1}) \cdot \y_t^{i-1}\sigma_t^i\ \leq\ \sum_{t=1}^T -\nabla \ell_t(\y_t^{i-1}) \cdot \y_t^{i-1} \frac{1}{W} + 2L_B B\sqrt{T}.
	\end{equation}
	Next, since $f = \sum_{g \in S} w_g g$ with $w_g \geq 0$, we have
	\begin{equation} \label{eq:averaging}
		\frac{1}{W}\sum_{t=1}^T \nabla \ell_t(\y_t^i)\cdot f(\x_t) = \frac{1}{\sum_{g \in S} w_g}   \sum_{t=1}^T \sum_{g \in S} w_g \nabla \ell_t(\y_t^i)\cdot g(\x_t)\  \geq\ \min_{g \in S} \sum_{t=1}^T \nabla \ell_t(\y_t^i)\cdot g(\x_t).
	\end{equation}
	Let $g^\star \in \arg\min_{g \in S} \sum_{t=1}^T \nabla \ell_t(\y_t^i)\cdot g(\x_t)$. Since $\A^i$ is an online learning algorithm for $\F$ with regret bound $R(\cdot)$ for the $1$-Lipschitz linear loss functions $\ell_t^i(\y) = \frac{1}{L_B}\nabla \ell_t(\y_t^{i-1}) \cdot \y$, and $g^\star \in \F$, multiplying the regret bound (\ref{eq:regret}) by $L_B$ we have
	\begin{equation} \label{eq:wl-cd-stepwise}
		\sum_{t=1}^T \nabla \ell_t(\y_t^{i-1}) \cdot \A^i(\x_t) \leq \sum_{t=1}^T \nabla \ell_t(\y_t^{i-1}) \cdot g^\star(\x_t) + L_B R(T) \leq \frac{1}{W} \sum_{t=1}^T \nabla \ell_t(\y_t^{i-1})\cdot f(\x_t) + L_B R(T)	
	\end{equation}
	by (\ref{eq:averaging}). Now, we analyze how much excess loss is potentially introduced due to the projection in line 8. First, note that if $B = \eta N D$, then the projection has no effect since $(1 - \sigma^i_t\eta)\y_t^{i-1} + \eta \A^i(\x_t) \in \B^d(B)$, and in this case $\ell_t(\y_t^i)\ =\ \ell_t((1 - \sigma^i_t\eta)\y_t^{i-1} + \eta \A^i(\x_t))$. If $B < \eta N D$, then by the definition of $B$, $\eta \beta_B B^2 \geq \epsilon_B D$, and since $(1 - \sigma^i_t\eta)\y_t^{i-1} \in \B^d(B)$ and $\|\eta \A^i(\x_t))\| \leq \eta D$, and we have
	\[\ell_t(\y_t^i)\ =\ \ell_t(\Pi_B((1 - \sigma^i_t\eta)\y_t^{i-1} + \eta \A^i(\x_t)))\ \leq\ \ell_t((1 - \sigma^i_t\eta)\y_t^{i-1} + \eta \A^i(\x_t)) + \eta \epsilon_B D.\]
	In either case, we have
	\begin{equation} \label{eq:excess-loss}
		\ell_t(\y_t^i)\ \leq\ \ell_t((1 - \sigma^i_t\eta)\y_t^{i-1} + \eta \A^i(\x_t)) + \eta^2 \beta_B B^2.
	\end{equation}

	We now move to the main part of the analysis. Define for $i = 0, 1, 2, \ldots, N$, $\Delta_i := \sum_{t=1}^T \ell_t(\y_t^i) - \ell_t(f(\x_t))$. We have
	\begin{align*}
		\Delta_i\ &\leq\ \left[\sum_{t=1}^T \ell_t((1 - \sigma_t^i\eta)\y_t^{i-1} + \eta \A^i(\x_t)) - \ell_t(f(\x_t))\right] + \eta^2\beta_B B^2 T \\
		&\leq\ \Delta_{i-1} + \left[\sum_{t=1}^T \eta \nabla \ell_t(\y_t^{i-1})\cdot (\A^i(\x_t) - \sigma_t^i\y_t^{i-1}) + \frac{\beta_B\eta^2}{2}\|\A^i(\x_t) - \sigma_t^i\y_t^{i-1}\|^2\right] + \eta^2\beta_B B^2 T \\
		&\ (\text{by $\beta_B$-smoothness})\\
		&\leq\ \Delta_{i-1}\ + \left[\sum_{t=1}^T \frac{\eta}{W} \nabla \ell_t(\y_t^{i-1})\cdot (f(\x_t) - \y_t^{i-1})\right] + 3\eta^2 \beta_B B^2T + \eta L_B R(T) + 2\eta L_B B\sqrt{T} \\
		&\ (\text{by (\ref{eq:ogd-regret}), (\ref{eq:wl-cd-stepwise}) and the fact that $\|\A^i(\x_t) - \sigma_t^i\y_t^{i-1}\| \leq D + B \leq 2B$})\\
		&\leq\ \left(1 - \frac{\eta}{W}\right)\Delta_{i-1} + 3\eta^2\beta_B B^2 T + \eta L_B R(T) + 2\eta L_B B\sqrt{T},
	\end{align*}
	since, by convexity of $\ell_t$ we have $\ell_t(\y_t^{i-1}) +  \nabla \ell(\y_t^{i-1}) \cdot (f(\x_t) - \y_t^{i-1}) \leq \ell_t(f(\x_t))$. Applying the above bound iteratively, we get
	\begin{align*}
		\Delta_N\ &\leq\ \left(1 - \frac{\eta}{W}\right)^N\!\!\!\Delta_0 + \sum_{i=1}^N \left(1 - \frac{\eta}{W}\right)^{i-1} \cdot (3\eta^2\beta_B B^2 T + \eta L_B R(T) + 2\eta L_B B\sqrt{T})\\
		&\leq\ \left(1 - \frac{\eta}{W}\right)^N\!\!\!\Delta_0 + 3\eta\beta_B B^2 W T + L_B WR(T) + 2 L_B B W\sqrt{T}.
	\end{align*}
	This completes the proof.
\end{proof}

\section{Variants of the boosting algorithms}

Our boosting algorithms and the analysis are considerably flexible: it is easy to modify the algorithms to work with a different (and perhaps more natural) kind of base learner which does greedy fitting, or incorporate a scaling of the base functions which improves performance. Also, when specialized to the batch setting, our algorithms provide better convergence rates than previous work.

\subsection{Fitting to actual loss functions}

The choice of an online {\em linear} learning algorithm over the base function class in our algorithms was made to ease the analysis. In practice, it is more common to have an online algorithm which produce predictions with comparable accuracy to the best function in hindsight for the {\em actual} sequence of loss functions. In particular, a common heuristic in boosting algorithms such as the original gradient boosting algorithm by \citet{Friedman01} or the matching pursuit algorithm of \citet{MallatZh93} is to build a linear combination of base functions by iteratively augmenting the current linear combination via greedily choosing a base function and a step size for it that minimizes the loss with respect to the residual label. Indeed, the boosting algorithm of \citet{ZhangYu05} also uses this kind of greedy fitting algorithm as the base learner.

In the online setting, we can model greedy fitting as follows. We first fix a step size $\alpha \geq 0$ in advance. Then, in each round $t$, the base learner $\A$ receives not only the example $\x_t$, but also an {\em offset} $\y'_t \in \R^d$ for the prediction, and produces a prediction $\A(\x_t) \in \R^d$, after which it receives the loss function $\ell_t$ and suffers loss $\ell_t(\y'_t + \alpha \A(\x_t))$. The predictions of $\A$ satisfy
\[ \sum_{t=1}^T \ell_t(\y'_t + \alpha \A(\x_t))\ \leq\ \inf_{f \in \F} \sum_{t=1}^T \ell_t(\y'_t + \alpha f(\x_t)) + R(T),\]
where $R$ is the regret. Our algorithms can be made to work with this kind of base learner as well. The details can be found in Section \ref{sec:actual-losses} of the \supplementary.

\subsection{Improving the regret bound via scaling}
\label{sec:main-scaling}

Given an online linear learning algorithm $\A$ over the function class $\F$ with regret $R$, then for any scaling parameter $\lambda > 0$, we trivially obtain an online linear learning algorithm, denoted $\lambda \A$, over a $\lambda$-scaling of $\F$, viz. $\lambda \F := \{\lambda f\ |\ f \in \F\}$, simply by multiplying the predictions of $\A$ by $\lambda$. The corresponding regret scales by $\lambda$ as well, i.e. it becomes $\lambda R$.

The performance of Algorithm~\ref{alg:spanboost} can be improved by using such an online linear learning algorithm over $\lambda \F$ for a suitably chosen scaling $\lambda \geq 1$ of the function class $\F$. The regret bound from Theorem~\ref{thm:span} improves because the $1$-norm of $f$ measured with respect to $\lambda \F$, i.e. $\|f\|'_1 = \max\{1, \frac{\|f\|_1}{\lambda}\}$, is smaller than $\|f\|_1$, but degrades because the parameter $B' = \min\{\eta N \lambda D,\ \inf \{b \geq \lambda D:\ \eta \beta_b b^2 \geq \epsilon_b \lambda D\}\}$ is larger than $B$. But, as detailed in Section \ref{sec:scaling} of the \supplementary, in many situations the improvement due to the former compensates for the degradation due to the latter, and overall we can get improved regret bounds using a suitable value of $\lambda$.

\subsection{Improvements for batch boosting}

Our algorithmic technique can be easily specialized and modified to the standard batch setting with a fixed batch of training examples and a base learning algorithm operating over the batch, exactly as in \citep{ZhangYu05}. The main difference compared to the algorithm of \citep{ZhangYu05} is the use of the $\sigma$ variables to scale the coefficients of the weak hypotheses appropriately. While a seemingly innocuous tweak, this allows us to derive analogous bounds to those of \citet{ZhangYu05} on the optimization error that show that our boosting algorithm converges exponential faster. A detailed comparison can be found in Section~\ref{sec:batch-boosting} of the \supplementary.

\section{Experimental Results}
\label{sec:experiments}
Is it possible to boost in an online fashion in practice with real base learners? To study this question, we implemented and evaluated 
Algorithms~\ref{alg:spanboost} and \ref{alg:chboost} 
within the Vowpal Wabbit (VW) open source machine learning system~\cite{VW}.
The three online base learners used were VW's default linear learner
(a variant of stochastic gradient descent), two-layer sigmoidal neural networks with 10 hidden units, and regression stumps.

Regression stumps were implemented by doing stochastic gradient descent on 
each individual feature, and predicting with the best-performing 
non-zero valued feature in the current example.

All experiments were done on a collection of 14 publically available
regression and classification datasets (described in Section~\ref{app:datasets} in the \supplementary) using squared loss.  The only parameters tuned were the learning rate and the number of weak learners, as well as the step size parameter for 
Algorithm~\ref{alg:spanboost}.
Parameters were tuned based on progressive validation loss on half 
of the dataset; reported is propressive validation loss
on the remaining half. Progressive validation is a standard online 
validation technique, where each training example is used for testing 
before it is used for updating the model~\cite{Blum:1999}.

The following table reports the average and the median, over the datasets, 
relative improvement in squared loss over the respective base learner. Detailed results can be found in Section~\ref{app:datasets} in the \supplementary.

  \begin{tabular}{lcccc}
    \toprule
    \multirow{2}{*}{Base learner} &
      \multicolumn{2}{c}{Average relative improvement} &
      \multicolumn{2}{c}{Median relative improvement} \\
      & Algorithm~\ref{alg:spanboost} 
      & Algorithm~\ref{alg:chboost} 
      & Algorithm~\ref{alg:spanboost} 
      & Algorithm~\ref{alg:chboost} \\
      \midrule
    SGD &  1.65\% & 1.33\% & 0.03\% & 0.29\% \\
    Regression stumps & 20.22\% & 15.9\% & 10.45\% & 13.69\% \\
    Neural networks & 7.88\% & 0.72\% & 0.72\% & 0.33\%  \\
    \bottomrule
  \end{tabular}

Note that both SGD (stochastic gradient descent) and neural networks are already very strong learners. Naturally, boosting is much more effective for regression stumps, which is a weak base learner.

\section{Conclusions and Future Work}

In this paper we generalized the theory of boosting for regression problems to the online setting and provided online boosting algorithms with theoretical convergence guarantees. Our algorithmic technique also improves convergence guarantees for batch boosting algorithms. We also provide experimental evidence that our boosting algorithms do improve prediction accuracy over commonly used base learners in practice, with greater improvements for weaker base learners. The main remaining open question is whether the boosting algorithm for competing with the span of the base functions is optimal in any sense, similar to our proof of optimality for the the boosting algorithm for competing with the convex hull of the base functions.

\bibliography{ref}

\begin{thebibliography}{25}
\providecommand{\natexlab}[1]{#1}
\providecommand{\url}[1]{\texttt{#1}}
\expandafter\ifx\csname urlstyle\endcsname\relax
  \providecommand{\doi}[1]{doi: #1}\else
  \providecommand{\doi}{doi: \begingroup \urlstyle{rm}\Url}\fi

\bibitem[Bartlett and Traskin(2007)]{BartlettTr07}
Peter~L. Bartlett and Mikhail Traskin.
\newblock {AdaBoost} is consistent.
\newblock \emph{Journal of Machine Learning Research}, 8:\penalty0 2347--2368,
  2007.

\bibitem[Beygelzimer et~al.(2015)Beygelzimer, Kale, and Luo]{BKL15}
Alina Beygelzimer, Satyen Kale, and Haipeng Luo.
\newblock Optimal and adaptive algorithms for online boosting.
\newblock In \emph{Proceedings of the 32nd International Conference on Machine
  Learning}, 2015.

\bibitem[Blum et~al.(1999)Blum, Kalai, and Langford]{Blum:1999}
Avrim Blum, Adam Kalai, and John Langford.
\newblock Beating the hold-out: Bounds for k-fold and progressive
  cross-validation.
\newblock In \emph{Proceedings of the Twelfth Annual Conference on
  Computational Learning Theory}, COLT '99, pages 203--208, 1999.

\bibitem[Chen et~al.(2012)Chen, Lin, and Lu]{ChenLiLu12}
Shang-Tse Chen, Hsuan-Tien Lin, and Chi-Jen Lu.
\newblock {An Online Boosting Algorithm with Theoretical Justifications}.
\newblock In \emph{Proceedings of the 29th International Conference on Machine
  Learning}, 2012.

\bibitem[Chen et~al.(2014)Chen, Lin, and Lu]{ChenLiLu14}
Shang-Tse Chen, Hsuan-Tien Lin, and Chi-Jen Lu.
\newblock {Boosting with Online Binary Learners for the Multiclass Bandit
  Problem}.
\newblock In \emph{Proceedings of the 31st International Conference on Machine
  Learning}, 2014.

\bibitem[Collins et~al.(2000)Collins, Schapire, and Singer]{CollinsScSi00}
Michael Collins, Robert~E. Schapire, and Yoram Singer.
\newblock Logistic regression, {AdaBoost} and {Bregman} distances.
\newblock In \emph{Proceedings of the Thirteenth Annual Conference on
  Computational Learning Theory}, 2000.

\bibitem[Duffy and Helmbold(2002)]{DuffyHe02}
Nigel Duffy and David Helmbold.
\newblock Boosting methods for regression.
\newblock \emph{Machine Learning}, 47\penalty0 (2/3):\penalty0 153--200, 2002.

\bibitem[Frank and Wolfe(1956)]{FW}
Marguerite Frank and Philip Wolfe.
\newblock An algorithm for quadratic programming.
\newblock \emph{Naval Res. Logis. Quart.}, 3:\penalty0 95--–110, 1956.

\bibitem[Freund and Schapire(1997)]{FreundSc97}
Yoav Freund and Robert~E. Schapire.
\newblock A decision-theoretic generalization of on-line learning and an
  application to boosting.
\newblock \emph{Journal of Computer and System Sciences}, 55\penalty0
  (1):\penalty0 119--139, August 1997.

\bibitem[Friedman(2001)]{Friedman01}
Jerome~H. Friedman.
\newblock Greedy function approximation: {A} gradient boosting machine.
\newblock \emph{Annals of Statistics}, 29\penalty0 (5), October 2001.

\bibitem[Grabner and Bischof(2006)]{GrabnerBi06}
Helmut Grabner and Horst Bischof.
\newblock On-line boosting and vision.
\newblock In \emph{CVPR}, volume~1, pages 260--267, 2006.

\bibitem[Grabner et~al.(2008)Grabner, Leistner, and Bischof]{GrabnerLeBi08}
Helmut Grabner, Christian Leistner, and Horst Bischof.
\newblock Semi-supervised on-line boosting for robust tracking.
\newblock In \emph{ECCV}, pages 234--247, 2008.

\bibitem[Hastie and Robet~Tibshirani(1990)]{HastieTibshirani}
Trevor Hastie and R.~J Robet~Tibshirani.
\newblock \emph{Generalized Additive Models}.
\newblock Chapman and Hall, 1990.

\bibitem[Hastie et~al.(2001)Hastie, Tibshirani, and Friedman]{HastieTiFr01}
Trevor Hastie, Robert Tibshirani, and Jerome Friedman.
\newblock \emph{The Elements of Statistical Learning: Data Mining, Inference,
  and Prediction}.
\newblock Springer Verlag, 2001.

\bibitem[Hazan and Kale(2014)]{HKsvm}
Elad Hazan and Satyen Kale.
\newblock Beyond the regret minimization barrier: optimal algorithms for
  stochastic strongly-convex optimization.
\newblock \emph{Journal of Machine Learning Research}, 15\penalty0
  (1):\penalty0 2489--2512, 2014.

\bibitem[Hazan et~al.(2007)Hazan, Agarwal, and Kale]{HazanAgKa07}
Elad Hazan, Amit Agarwal, and Satyen Kale.
\newblock Logarithmic regret algorithms for online convex optimization.
\newblock \emph{Machine Learning}, 69\penalty0 (2-3):\penalty0 169--192, 2007.

\bibitem[Liu and Yu(2007)]{LiuYu07}
Xiaoming Liu and Ting Yu.
\newblock Gradient feature selection for online boosting.
\newblock In \emph{ICCV}, pages 1--8, 2007.

\bibitem[Mallat and Zhang(1993)]{MallatZh93}
St\'ephane~G. Mallat and Zhifeng Zhang.
\newblock Matching pursuits with time-frequency dictionaries.
\newblock \emph{IEEE Transactions on Signal Processing}, 41\penalty0
  (12):\penalty0 3397--3415, December 1993.

\bibitem[Mason et~al.(2000)Mason, Baxter, Bartlett, and Frean]{MasonBaBaFr00}
Llew Mason, Jonathan Baxter, Peter Bartlett, and Marcus Frean.
\newblock Boosting algorithms as gradient descent.
\newblock In \emph{Advances in Neural Information Processing Systems 12}, 2000.

\bibitem[Oza and Russell(2001)]{OzaRu01}
Nikunj~C. Oza and Stuart Russell.
\newblock Online bagging and boosting.
\newblock In \emph{Eighth International Workshop on Artificial Intelligence and
  Statistics}, pages 105--112, 2001.

\bibitem[Schapire and Freund(2012)]{SchapireFr12}
Robert~E. Schapire and Yoav Freund.
\newblock \emph{Boosting: Foundations and Algorithms}.
\newblock MIT Press, 2012.

\bibitem[Telgarsky(2013)]{Telgarsky13}
Matus Telgarsky.
\newblock Boosting with the logistic loss is consistent.
\newblock In \emph{Proceedings of the 26th Annual Conference on Learning
  Theory}, 2013.

\bibitem[VW()]{VW}
VW.
\newblock URL \url{https://github.com/JohnLangford/vowpal_wabbit/}.

\bibitem[Zhang and Yu(2005)]{ZhangYu05}
Tong Zhang and Bin Yu.
\newblock Boosting with early stopping: Convergence and consistency.
\newblock \emph{Annals of Statistics}, 33\penalty0 (4):\penalty0 1538--1579,
  2005.

\bibitem[Zinkevich(2003)]{Zinkevich03}
Martin Zinkevich.
\newblock Online convex programming and generalized infinitesimal gradient
  ascent.
\newblock In \emph{Proceedings of the Twentieth International Conference on
  Machine Learning}, 2003.

\end{thebibliography}
\bibliographystyle{plainnat}

\newpage

\begin{center}
\bf \Large Supplementary material for ``Online Gradient Boosting''
\end{center}

\appendix

\section{Variants of the boosting algorithms}
\label{sec:extensions}

In this section we provide the omitted details of two variants of our boosting algorithms: (a) a variant that work with a different kind of base learner which does greedy fitting, and (b) a variant that incorporates a scaling of the base functions to improves performance. We also show how our algorithmic technique can be used to improve the convergence speed for batch boosting.

\subsection{Fitting to actual loss functions}
\label{sec:actual-losses}

The choice of an online {\em linear} learning algorithm over the base function class in our algorithms was made to ease the analysis. In practice, it is more common to have an online algorithm which produce predictions with comparable accuracy to the best function in hindsight for the {\em actual} sequence of loss functions. In particular, a common heuristic in boosting algorithms such as the original gradient boosting algorithm by \citet{Friedman01} or the matching pursuit algorithm of \citet{MallatZh93} is to build a linear combination of base functions by iteratively augmenting the current linear combination by greedily choosing a base function and a step size for it that minimizes the loss with respect to the residual label. Indeed, the boosting algorithm of \citet{ZhangYu05} also uses this kind of greedy fitting algorithm as the base learner.

In the online setting, we can model greedy fitting as follows. We first fix a step size $\alpha \geq 0$ in advance. Then, in each round $t$, the base learner $\A$ receives not only the example $\x_t$, but also an {\em offset} $\y'_t \in \R^d$ for the prediction, and produces a prediction $\A(\x_t) \in \R^d$, after which it receives the loss function $\ell_t$ and suffers loss $\ell_t(\y'_t + \alpha \A(\x_t))$. The predictions of $\A$ satisfy
\[ \sum_{t=1}^T \ell_t(\y'_t + \alpha \A(\x_t))\ \leq\ \inf_{f \in \F} \sum_{t=1}^T \ell_t(\y'_t + \alpha f(\x_t)) + R(T),\]
where $R$ is the regret. We now describe how our algorithms can be made to work with this kind of base learner as well.

Assume that for some known parameter $B > 0$, we have $\|\y'_t\| \leq B$, for all $t$. Let $B' = B + \alpha D$, and assume that the loss functions $\ell_t$ are $L_{B'}$ Lipschitz and $\beta_{B'}$ smooth on $\B^d(B')$. Then using the convexity and smoothness of the loss functions, we have $\ell_t(\y'_t + \alpha \A(\x_t)) \geq \ell_t(\y'_t ) + \alpha \nabla\ell_t(\y'_t) \cdot \A(\x_t)$ and $\ell_t(\y'_t + \alpha f(\x_t)) \leq \ell_t(\y'_t ) + \alpha \nabla \ell_t(\y'_t) \cdot f(\x_t) + \frac{\beta_{B'}\alpha^2}{2} \|f(\x_t)\|^2$.
Plugging these bounds into the above regret bound we get, for any $f \in \F$,
\[ \sum_{t=1}^T \nabla \ell_t(\y'_t) \cdot \A(\x_t)\ \leq\ \sum_{t=1}^T \left( \nabla \ell_t(\y'_t) \cdot f(\x_t) + \frac{\beta_{B'}}{2}\alpha \|f(\x_t)\|^2 \right)+ \frac{1}{\alpha}R(T).\]
Since $\|f(\x_t)\| \leq D$, setting $\alpha = \sqrt{\frac{2R(T)}{\beta_{B'} D^2T}}$, we conclude that
\begin{equation}
	\label{eq:nonscaling-regret}
	\sum_{t=1}^T \nabla \ell_t(\y'_t) \cdot \A(\x_t)\ \leq\ \sum_{t=1}^T \nabla \ell_t(\y'_t) \cdot f(\x_t) + \sqrt{2\beta_{B'} D^2 T R(T)}.
\end{equation}
This regret bound is sublinear in $T$ if $R(T)$ is sublinear. We can obtain a better regret bound if we assume that $R(T)$ scales linearly with $\alpha$: this is a natural assumption since the functions $\ell_t(\y_t' + \alpha \y)$ are $\alpha L_{B'}$ Lipschitz in the prediction $\y$. In this case, the regret bound $R(T) = \alpha R'(T)$ for some fixed $R': \mathbb{N} \rightarrow \R_+$ indepedent of $\alpha$, and we can choose $\alpha = \frac{2R'(T)}{\beta_{B'} D^2T}$ so that
\begin{equation}
	\label{eq:scaling-regret}
	\sum_{t=1}^T \nabla \ell_t(\y'_t) \cdot \A(\x_t)\ \leq\ \sum_{t=1}^T \nabla \ell_t(\y'_t) \cdot f(\x_t) + 2R'(T).
\end{equation}

Either the bound (\ref{eq:nonscaling-regret}) or (\ref{eq:scaling-regret}) suffices for the analysis of our boosting algorithms to go through: to use this kind of base learner $\A$, we again keep $N$ copies $\A^1, \A^2, \ldots, \A^N$ with a suitably chosen step size $\alpha$, and simply change line 11 of Algorithm~\ref{alg:chboost} and line 13 of Algorithm~\ref{alg:spanboost} to pass the offset $\y'_t = \y_t^{i-1}$ to $\A^i$.

\subsection{Improving the regret bound via scaling}
\label{sec:scaling}

Given an online linear learning algorithm $\A$ over the function class $\F$ with regret $R$, then for any scaling parameter $\lambda > 0$, we trivially obtain an online linear learning algorithm, denoted $\lambda \A$, over a $\lambda$-scaling of $\F$, viz. $\lambda \F := \{\lambda f\ |\ f \in \F\}$, simply by multiplying the predictions of $\A$ by $\lambda$. The corresponding regret scales by $\lambda$ as well, i.e. it becomes $\lambda R$.

The performance of Algorithm~\ref{alg:spanboost} can be improved by using such an online linear learning algorithm over $\lambda \F$ for a suitably chosen scaling $\lambda \geq 1$ of the function class $\F$. Let $\|f\|'_1 = \max\{1, \frac{\|f\|_1}{\lambda}\}$ be the $1$-norm of $f$ measured with respect to $\lambda \F$, and $B' = \min\{\eta N \lambda D,\ \inf \{b \geq \lambda D:\ \eta \beta_b b^2 \geq \epsilon_b \lambda D\}\}$. Then we immediately get the following corollary of Theorem~\ref{thm:span}:
\begin{corollary}
	For any $f \in \Span(\F)$, let $\Delta_0 = \sum_{t=1}^T \ell_t(0) - \ell_t(f(\x_t))$. Algorithm~\ref{alg:spanboost}, using $\lambda \A$ as the online linear algorithm over $\lambda \F$, is an online learning algorithm for $\Span(\F)$ for losses in $\Ccvx$ with the following regret bound for any $f \in \Span(\F)$:

\[ R'_f(T) \leq \left(1 - \frac{\eta}{\|f\|'_1}\right)^N\!\!\!\Delta_0 + 3\eta\beta_{B'} {B'}^2 \|f\|'_1 T + L_{B'} \|f\|'_1 \lambda R(T) + 2L_{B'} {B'} \|f\|'_1\sqrt{T}.\]
\end{corollary}

Choosing large values of $\lambda$ implies that $\|f\|'_1$ can be significantly smaller than $\|f\|_1$. But $B'$ becomes bigger than $B$, and correspondingly, the parameters $\beta_{B'}$ and $L_{B'}$ become bigger than $\beta_B$ and $L_B$ respectively. Also, the (lower order) dependence on the regret term $R(T)$ increases by a factor of $\lambda$.

However, it turns out (see Section~\ref{sec:params}) that in several common applications of the algorithm, $B'$ can be set to be equal to $B$ or the increase from $B$ is a very slow growing function of $\lambda$, such as $\log(\lambda)$. In such situations choosing larger values of $\lambda$ leads to improvement in the higher order terms of the regret bound, while making the lower order term (i.e. $L_{B'} \|f\|'_1 \lambda R(T)$) worse; overall the regret bound can be improved by choosing a suitably large scaling factor $\lambda$ to balance between the two.

\subsection{Improvements for batch boosting}
\label{sec:batch-boosting}

Our algorithmic technique can be used to improve convergence speed for batch boosting as well, in the setup considered by \citet{ZhangYu05}. Since the focus of this paper is on online boosting, we give a high level comparison of the bounds here, making some simplifying assumptions to ease the technical details, using our notation as much as possible.

In the setup of \citet{ZhangYu05}, we have a base set of real valued functions $\F$, which we assume is symmetric and contains the zero function, $\bzero$. Then $\Span(\F)$ is a linear function space, and let $\|\cdot\|$ be some norm defined on $\Span(\F)$. For clarity of presentation, we assume that for any $f \in \F$, we have $\|f\| \leq 1$. This implies that for any $f \in \Span(\F)$, we have $\|f\| \leq \|f\|_1$. 

The goal is to minimize a given convex functional $\ell: \Span(\F) \rightarrow \R$ over its domain, $\Span(\F)$. We assume, for simplicity, that $\ell$ is $\beta$-smooth over $\Span(\F)$ under the norm $\|\cdot\|$, i.e. for any $f, f' \in \Span(\F)$, we have
\[ \ell(f')\ \leq\ \ell(f) + \nabla \ell(f) \cdot (f' - f) + \frac{\beta}{2}\|f - f'\|^2.\]

\citet{ZhangYu05} assume\footnote{This is a slight simplification of the base learning algorithm considered in \citep{ZhangYu05}, which also performs a search over the step size $\eta$. Also, the analysis in \citep{ZhangYu05} allows some optimization error for the base learning algorithm; to simplify the comparison we assume this error is $0$.} that we have access to a base learning algorithm $\A$ that, given any $f \in \Span(\F)$ and a step size $\eta \geq 0$ can find a function $g \in \F$ that minimizes $\ell(f + \eta g)$. We denote the output of $\A$ by $\A(f, \eta)$. 

Given such a base learning algorithm, and a sequence of step sizes $\eta_1, \eta_2, \ldots$, the boosting algorithm of \citet{ZhangYu05} computes a sequence of functions $f_0 , f_1, f_2, \ldots \in \Span(\F)$ via greedy fitting as follows: $f_0$ is set to $\bzero$, and for any $i \geq 1$,
\[ f_i\ :=\ f_{i-1} + \eta_i \A(f_{i-1}, \eta_i).\]
Define $s_0 = 1$ and $s_i = s_{i-1} + \eta_i$ for any $i \geq 1$.

For any $f \in \Span(f)$, for $i = 1, 2, \ldots$, let $\Delta_i = \ell(f_i) - \ell(f)$ denote the optimization errors of the function $f_i$. \citet{ZhangYu05} prove that for any $N \in \mathbb{N}$, we have
\begin{equation}
	\label{eq:ZY-bound}
	\Delta_N\ \leq\ \frac{s_0 + \|f\|_1}{s_N + \|f\|_1} \Delta_0 + \sum_{i=1}^N \frac{s_i + \|f\|_1}{s_N + \|f\|_1} \cdot \frac{\beta}{2}\eta_i^2.
\end{equation}

Using the techniques in this paper, we can define a different boosting algorithm which works as follows. Given the same sequence of step sizes $\eta_1, \eta_2, \ldots$ as above, we set $f_0 = \bzero$, and for any $i \geq 1$, 
\[ f_i\ :=\ (1 - \sigma_i \eta_i) f_{i-1} + \eta_i \A(f_{i-1}, \eta_i),\]
where 
\[
\sigma_i\ :=\ \begin{cases}
	1 & \text{ if } \nabla \ell(f_{i-1}) \cdot f_{i-1} \geq 0\\
	0 & \text{ otherwise.}
\end{cases}
\]
We can analyze this algorithm along the lines of the proof of Theorem~\ref{thm:span}. First, let $g_i = \A(f_{i-1}, \eta_i)$. Then for $g \in \F$, we have $\ell(f_{i-1} + \eta_i g_i) \leq \ell(f_{i-1} + \eta_i g)$, and by the convexity and $\beta$-smoothness of $\ell$, we conclude that
\[ \nabla \ell(f_{i-1})\cdot g_i\ \leq\ \nabla \ell(f_{i-1})\cdot g + \frac{\beta}{2}\eta_i.\]
Using this fact and following the proof of Theorem~\ref{thm:span}, we get the following bound on the optimization error $\Delta_i = \ell(f_i) - \ell(f)$ of the function $f_i$:
\begin{equation}
	\label{eq:BHKL-bound}
	\Delta_N\ \leq\ \exp\left(-\frac{s_N - s_0}{\|f\|_1}\right)\Delta_0 + \sum_{i=1}^N \exp\left(-\frac{s_N - s_i}{\|f\|_1}\right) \cdot \frac{\beta}{2}\eta_i^2 (s_i^2 + 1).
\end{equation}

We can compare our bound (\ref{eq:BHKL-bound}) to the bound (\ref{eq:ZY-bound}) of \citet{ZhangYu05}, by comparing corresponding terms in the bound. For each term, we can calculate how large $s_N$ needs to be for the term to be reduced to less than some given bound $\epsilon$. To reduce the first term to less than $\epsilon$ our algorithm needs $s_N \geq \|f\|_1 \log(\frac{\Delta_0}{\epsilon}) + s_0$, whereas the algorithm of \citet{ZhangYu05} needs $s_N \geq (\frac{\Delta_0}{\epsilon})(s_0 + \|f\|_1) - \|f\|_1$. As for the second term, to reduce the $i$-th term in the sum to less than $\epsilon$, our algorithm needs $s_N \geq \|f\|_1 \log(\frac{\beta \eta_i^2 (s_i^2+1)}{2\epsilon}) + s_i$, whereas the algorithm of \citet{ZhangYu05} needs $s_N \geq (\frac{\beta\eta_i^2}{2\epsilon})(s_i + \|f\|_1) - \|f\|_1$. Since in either case, the dependence on $\epsilon$ is $\log(\frac{1}{\epsilon})$ for our algorithm, whereas it is $\frac{1}{\epsilon}$ for the algorithm of \citet{ZhangYu05}, we conclude that our algorithm converges exponentially faster.

\section{Description of Data Sets and Detailed Experimental Results}
\label{app:datasets}
The datasets come from the UCI repository and various KDD Cup challenges.
Below, $d$ is the number of unique features in the dataset,
and $s$ is the average number of features per example.

\begin{small}
\begin{tabular}{crcccc}
\toprule
Dataset & Number of & Total number of & Average number of & Task & Label \\
        & instances & features & features per example & & range \\
\midrule
a9a &  48,841  & 123 & 14 & classification & $[-1,1]$
\\
abalone & 4,177 & 10 & 9 & regression & $[1,29]$
\\
activity &  165,632  &  20 & 18.5 &  classification & $[-1,1]$
\\
adult &  48,842  &  105 & 12 & classification & $[0,1]$
\\
bank & 45,211	& 45 & 15 & classification & $[-1,1]$
\\
cal\_housing & 20,640 & 9 & 9 & regression & $[0,1]$
\\
casp & 45,730 & 10 & 10 & regression & $[0,1]$
\\
census &  299,284  & 401 &  32 & classification & $[-1,1]$
\\
covtype &  581,011  & 54 &  12 & classification & $[-1,1]$
\\
kddcup04 (phy) &  50,000 & 74 & 32 & classification & $[0,1]$
\\
letter &  20,000  &  16 & 15.6 & classification & $[-1,1]$ 
\\
shuttle & 43,500 & 9 & 8 & classification & $[-1,1]$
\\
slice & 53,500 & 385 & 135 & regression & $[0,1]$
\\
year & 463,715 & 90 & 90 & regression & $[0,1]$
\\
\bottomrule
\end{tabular}
\end{small}

\bigskip
The following table provides the online squared losses summarized in Section~\ref{sec:experiments}.

\begin{small}
\begin{tabular}{l|ccc|ccc|ccc}
 \toprule
&  \multicolumn{3}{c|}{SGD} &  \multicolumn{3}{c|}{Regression stumps} &  \multicolumn{3}{c}{Neural Networks} \\
 \midrule
Dataset & Baseline & Alg \ref{alg:spanboost} & Alg \ref{alg:chboost} & Baseline & Alg \ref{alg:spanboost}  & Alg \ref{alg:chboost} & Baseline & Alg \ref{alg:spanboost}  & Alg  \ref{alg:chboost} \\
 \midrule
kddcup04/phy & 0.7475 & 0.7466 & 0.7470 & 0.9201 & 0.7733 & 0.7924 & 0.7441 & 0.7480 & 0.7446 \\
cal\_housing & 0.0094 & 0.0094 & 0.0104 & 0.0151 & 0.0138 & 0.0124 & 0.0096 & 0.0096 & 0.0107 \\
casp & 0.0632 & 0.0631 & 0.0630 & 0.0741 & 0.0741 & 0.0742 & 0.0639 & 0.0632 & 0.0631 \\
a9a & 0.4261 & 0.4283 & 0.4249 & 0.5749 & 0.5074 & 0.5758 & 0.4256 & 0.4266 & 0.4246 \\
abalone & 3.7263 & 3.7482 & 3.7154 & 6.7791 & 3.8273 & 4.2270 & 3.7380 & 3.7255 & 3.7212 \\
activity & 0.0334 & 0.0337 & 0.0316 & 0.4492 & 0.1454 & 0.3141 & 0.0192 & 0.0143 & 0.0186 \\
adult & 0.1055 & 0.1057 & 0.1056 & 0.1388 & 0.1261 & 0.1250 & 0.1081 & 0.1062 & 0.1081 \\
bank & 0.2971 & 0.2968 & 0.2973 & 0.3774 & 0.3240 & 0.3257 & 0.2962 & 0.2969 & 0.2969 \\
census & 0.1544 & 0.1545 & 0.1553 & 0.2073 & 0.1884 & 0.1789 & 0.1531 & 0.1531 & 0.1523 \\
covtype & 0.7256 & 0.7270 & 0.7286 & 0.7910 & 0.7986 & 0.7911 & 0.6807 & 0.6465 & 0.6757 \\
letter & 0.6441 & 0.5698 & 0.6108 & 0.7420 & 0.7087 & 0.7168 & 0.6542 & 0.5729 & 0.6108 \\
shuttle & 0.1616 & 0.1547 & 0.1577 & 0.8551 & 0.3678 & 0.4354 & 0.0760 & 0.0694 & 0.0802 \\
slice & 0.0076 & 0.0067 & 0.0065 & 0.0559 & 0.0362 & 0.0410 & 0.0054 & 0.0022 & 0.0044 \\
year & 0.0116 & 0.0119 & 0.0115 & 0.0152 & 0.0140 & 0.0141 & 0.0116 & 0.0119 & 0.0122 \\
 \bottomrule
\end{tabular}
\end{small}

\end{document}